\documentclass[11pt,a4paper]{article}
\usepackage{emnlp2021}

\usepackage[hyperref]{}
\usepackage{times}
\usepackage{latexsym}
\usepackage{amsmath}
\usepackage{subcaption}
\usepackage{algorithm}
\usepackage[algo2e, ruled]{algorithm2e} 

\usepackage{graphicx} 
\usepackage{booktabs} 
\usepackage{multirow} 
\usepackage{comment} 
\usepackage{amssymb} 
\usepackage{linguex} 
\usepackage{adjustbox} 
\usepackage{bm}
\usepackage{cleveref}

\newcommand{\R}{\mathbb{R}}
\newlength{\gap}
\newcommand{\matr}[1]{\mathbf{#1}}
\usepackage{microtype}

\newcommand{\method}{AlterRep}

\newcommand{\mask}{\textsc{[MASK]}}

\usepackage{amsthm}
\newtheorem{theorem}{Theorem}[section]

\newtheorem{claim}[theorem]{Claim}

\usepackage{todonotes}
\makeatletter
\newcommand*\iftodonotes{\if@todonotes@disabled\expandafter\@secondoftwo\else\expandafter\@firstoftwo\fi}  
\makeatother

\title{Counterfactual Interventions Reveal the Causal Effect of Relative Clause Representations on Agreement Prediction}

\author{Shauli Ravfogel$\footnotemark[1]$  \textsuperscript{1,2} \;\;\; Grusha Prasad$\thanks{~~Equal contribution.}$ \textsuperscript{3} \;\;\; Tal Linzen\textsuperscript{4} \;\;\; Yoav Goldberg\textsuperscript{1,2} \\
\textsuperscript{1}Computer Science Department, Bar Ilan University \\
\textsuperscript{2}Allen Institute for Artificial Intelligence \\
\textsuperscript{3}Cognitive Science Department, Johns Hopkins University \\
\textsuperscript{4}Department of Linguistics and Center for Data Science, New York University \\
 \tt{linzen@nyu.edu, grusha.prasad@jhu.edu} \\
 
 \tt{ \{shauli.ravfogel, yoav.goldberg\}@gmail.com}
 }

\usepackage{fancyhdr}
\usepackage{fancyheadings}
\begin{document}
\maketitle
\begin{abstract}
When language models process syntactically complex sentences, do they use their representations of syntax in a manner that is consistent with the grammar of the language? We propose \method, an intervention-based method to address this question. For any linguistic feature of a given sentence, \method\ generates counterfactual representations by altering how the feature is encoded, while leaving intact all other aspects of the original representation. By measuring the change in a model's word prediction behavior when these counterfactual representations are substituted for the original ones, we can draw conclusions about the causal effect of the linguistic feature in question on the model's behavior. We apply this method to study how BERT models of different sizes process relative clauses (RCs). We find that BERT variants use RC boundary information during word prediction in a manner that is consistent with the rules of English grammar; this RC boundary information generalizes to a considerable extent across different RC types, suggesting that BERT represents RCs as an abstract linguistic category. 
\end{abstract}

\setlength{\Exlabelwidth}{0.5em}
\setlength{\SubExleftmargin}{1.35em}

\section{Introduction}
\label{intro}

The success of neural language models, both in NLP tasks and as cognitive models, has fueled targeted evaluation of these models' word prediction accuracy on a range of syntactically complex constructions \cite{linzen2016assessing, gauthier2020syntaxgym, warstadt2020blimp, mueller2020crosslinguistic, marvin-linzen-2018-targeted}. 
What are the internal representations that support such sophisticated syntactic behavior? In this paper, we tackle this question using an intervention-based approach \cite{woodward2005making}. Our method, \method, is designed to study whether a model uses a particular linguistic feature in a manner which is consistent with the grammar of the language. The method involves two steps: first, it generates \textit{counterfactual}\footnote{We use the word \textit{counterfactual} as it is used when referring to \emph{counterfactual examples} \cite{DBLP:journals/corr/abs-2010-10596}: an altered version of an element that is similar to the original element in all aspects except one.} contextual
word representations by altering the neural network’s representation of the linguistic feature under consideration; and second, it characterizes the change in the model’s word prediction behaviour that results from replacing the original word representations with their counterfactual variants. If the resulting change in word prediction aligns with predictions from linguistic theory, we can infer that the model uses the feature under consideration in a manner consistent with the grammar of the language.

We demonstrate the utility of \method\ using relative clauses (RCs). According to the grammar of English, to correctly determine whether the masked verb in~\ref{ex:orc_agreement_mask} should be singular or plural, a model must recognize that the masked verb is outside the RC \emph{the officers love}, and should therefore agree with the subject of the main clause (\textit{the skater}, which is singular), rather than with the subject of the RC (\textit{the officers}, which is plural). 

\ex. \label{ex:orc_agreement_mask}The skater \textbf{the officers love} \mask\  happy. 

To investigate whether a neural model uses RC boundary representations as predicted by the grammar of English, we use \method\ to generate two counterfactual representations of the masked verb: one which encodes (incorrectly) the verb is \textit{inside} the RC, and another which encodes (correctly) that the verb is \textit{outside} the RC. Crucially, the difference between the counterfactual and original representations is \textit{minimal}: the aspects of the representation which do not encode information about RC boundaries remain unchanged. Therefore, if the model uses RC boundary information as dictated by the grammar of English---and if our method successfully identifies the way in which RC boundary information is represented by the model---we expect the incorrect counterfactual to cause the masked verb to incorrectly agree with the noun \textit{inside} the RC, and the correct counterfactual to cause agreement with the noun \textit{outside} the RC, correctly.

\begin{figure}
    \centering
    \includegraphics[width=0.99\columnwidth]{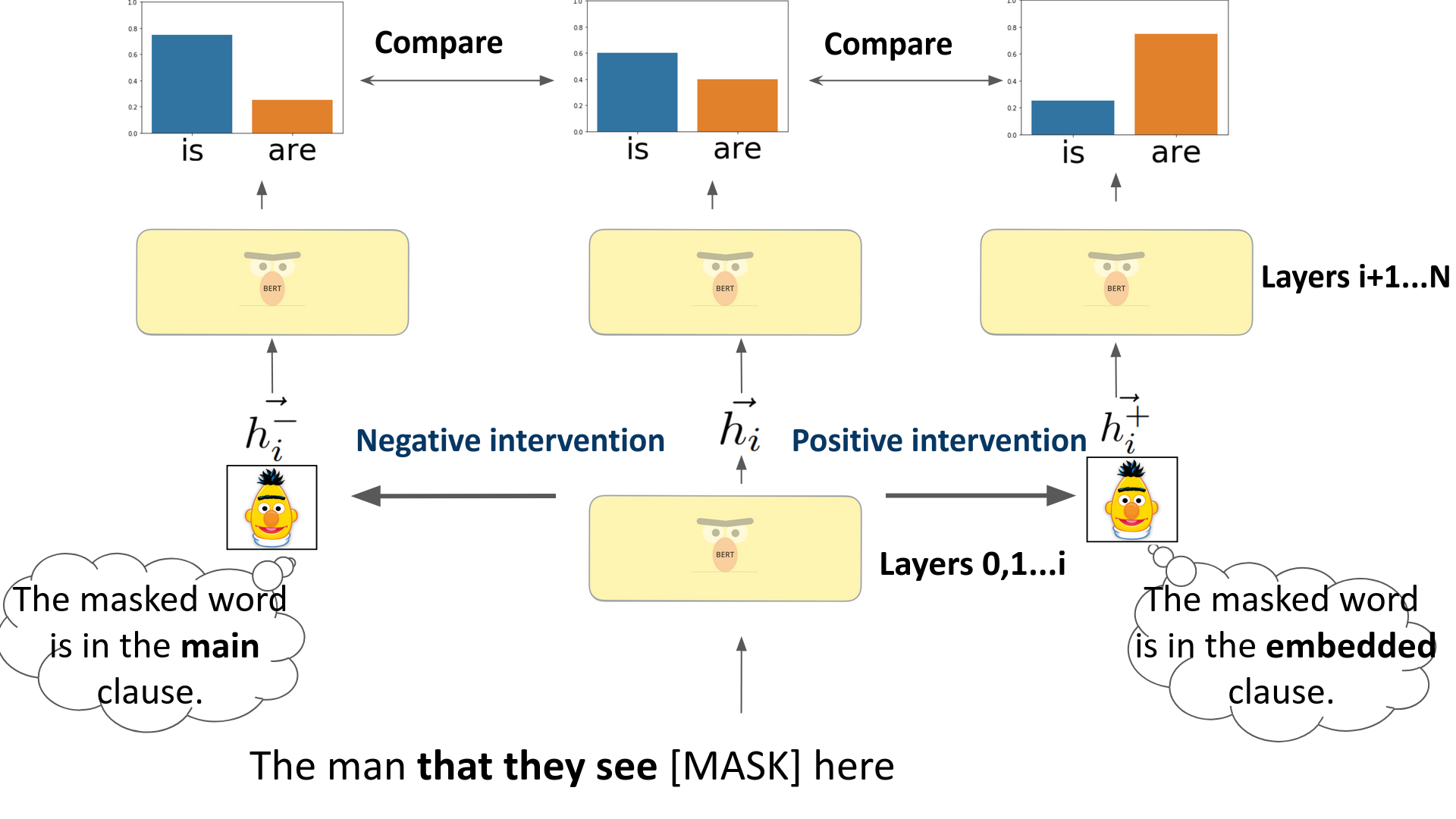}
    \caption{Causal analysis with counterfactual intervention. Given a representation $h$ of a masked word, we derive two new representations $h^{-}, h^{+}$ that differ in the information they contain with respect to a specific linguistic property. The predictions of the model over the counterfactual representations are compared with the original prediction $\hat{Y}$.}
    \label{fig:counterfactuals-graphical}
\end{figure}

We report experiments applying this logic to BERT variants of different sizes \cite{bert, DBLP:journals/corr/abs-1908-08962}. 
We found that while all layers of the BERT variants encoded information about RC boundaries, only the information encoded in the middle layers was \textit{used} in a manner consistent with the grammar of English. This contrast highlights the pitfalls of drawing behavioral conclusions from probing results alone, and motivates causal approaches such as \method.  

For BERT-base, we also found that counterfactual representations learned solely from one type of RC influenced the model's predictions in sentences containing \textit{other} RC types, suggesting that this model encodes information about RC boundaries in an abstract manner that generalizes across different RC types. Going beyond our case study of RC representations in BERT variants, we hope that future work can apply this method to test linguistically motivated hypotheses about a wide range of structures, tasks and models.


\section{Background}

\begin{table*}[t]
        \resizebox{\textwidth}{!}{
		\begin{tabular}{ll}
		\toprule
		\textbf{Abstract structure} & \textbf{Example} \\
		\midrule
		Unreduced Object RC (ORC) & The conspiracy that the employee welcomed divided the \textcolor{gray}{beautiful} country. \\
		Reduced Object RC (ORRC) & The conspiracy the employee welcomed divided the \textcolor{gray}{beautiful} country. \\
		Unreduced Passive RC (PRC) & The conspiracy that was welcomed by the employee divided the \textcolor{gray}{beautiful} country. \\
		Reduced Passive RC (PRRC) & The conspiracy welcomed by the employee divided the \textcolor{gray}{beautiful} country.\\
		Active Subject RC (SRC) & The employee that welcomed the conspiracy \textcolor{gray}{quickly} searched the building\textcolor{gray}{s}. \\
		\midrule
		P/OR(R)C-matched Coordination & The conspiracy welcomed the employee and divided the \textcolor{gray}{beautiful} country. \\
		SRC-matched Coordination & The employee welcomed the conspiracy and \textcolor{gray}{quickly} searched the building\textcolor{gray}{s}.\\

		\bottomrule
		
		\end{tabular}}
		
		\caption{\label{tab:ex}  Examples of sentences generated from the 5 RC structures, the 2 coordination structures. Elements which only occur in a subset of the examples are indicated in grey. This table is copied from \citet{prasad-etal-2019-using}.
		}

\end{table*}

\subsection{Relative clauses (RCs)}\label{rc_intro}
\label{sec:rc}

An RC is a subordinate clause that modifies a noun. The head of the RC needs to be interpreted twice---once in the main clause, and once inside the RC---but it is omitted from inside the RC, replaced by an unpronounced ``gap''. For example, in \ref{ex:orc}, the RC (in bold) describes the subject of the main clause \textit{the book}. Since \textit{the book} is the object of the embedded clause, we say that the gap is in the object position of the RC (indicated by underscores).

\ex. \label{ex:orc} The books \textbf{that my cousin likes \underline{\hspace{\gap}}} were interesting. (Object RC)

RCs can structurally differ from the Object RC in~\ref{ex:orc} in several ways: the overt complementizer \emph{that} can be excluded, as in~\ref{ex:orrc}; the gap can be in the subject instead of object position of the embedded clause, as in~\ref{ex:src}; 
and so on. The five types of RCs we consider in this paper are outlined in Table~\ref{tab:ex}.

\ex. \label{ex:orrc} The books \textbf{my cousin likes \underline{\hspace{\gap}}} were interesting. (Reduced Object RC)

\ex. \label{ex:src}  My cousin \textbf{that \underline{\hspace{\gap}} likes the books} was interesting. (Subject RC)


These differences do not affect the strategy that a system that follows the grammar of English should use to determine the number of the verb: regardless of the internal structure of the RC, a verb outside the RC should agree with the subject of the main clause, whereas a verb inside the RC should agree with the subject of the RC. 
Thus, a model that does not properly identify the boundaries of the RC will often predict a singular verb where a plural one is required, or vice versa.

\subsection{Iterative Null Space Projection (INLP)}
\label{sec:inlp}

INLP \cite{inlp} is a method for selectively identifying and removing user-defined concept features from a contextual representation. Let $T$ be a set of words-in-context, and let $H$ be the set of representations of $T$, such that $\vec{h_t} \in \mathbb{R}\strut^d$ is the contextualized representation of the word $t \in T$. Let $F$ be a linguistic feature that we hypothesize is encoded in $H$.
Given $H$ and the values $f_t$ of the feature $F$ for each word, INLP returns a set of $m$ linear classifiers, each of which predicts $F$ with above-chance accuracy. Each of these classifiers is a vector in $\R^d$, and corresponds to a direction in the representation space. The $m$ vectors can be arranged in a matrix $\matr{W}^{m \times d}$. Since the $m$ classifiers are mutually orthogonal, so are the rows of $\matr{W}$. Each linear classifier can be interpreted as defining a \textit{separating plane}, which is intended to partition the space, as much as possible, according to the values of the feature $F$. In our case, $F$ can take one of two values---whether or not a given word $t$ is in an RC---and each direction in $\matr{W}$ is intended to separate words that are inside RCs from words that are outside them.\footnote{In this paper, we make the simplifying assumption that sentences do not contain RCs that are nested within one another (cf. \citealt{lakretz2021}). To accommodate such sentences in future work, an integer feature could be used whose value would be $0$ if the word is outside any RC; $1$ if it is inside an RC of depth~$1$; $2$ if it is inside an RC of depth~$2$; and so on. As long as we specify a bound on the embedding depth, this feature would still be categorical, and a variant of our method could still be used.}

The \textit{feature subspace}---the space spanned by all the learned directions ($R=span(\matr{W})$)---is a subspace of the original representation space that contains information useful to linearly decode $F$ with high accuracy.
The orthogonal complement of $R$ (the \emph{null space}; $N$) is a subspace in which it is \emph{not} possible to predict $F$ with high accuracy.

\begin{figure}[t]
\centering
\includegraphics[width=0.99\columnwidth]{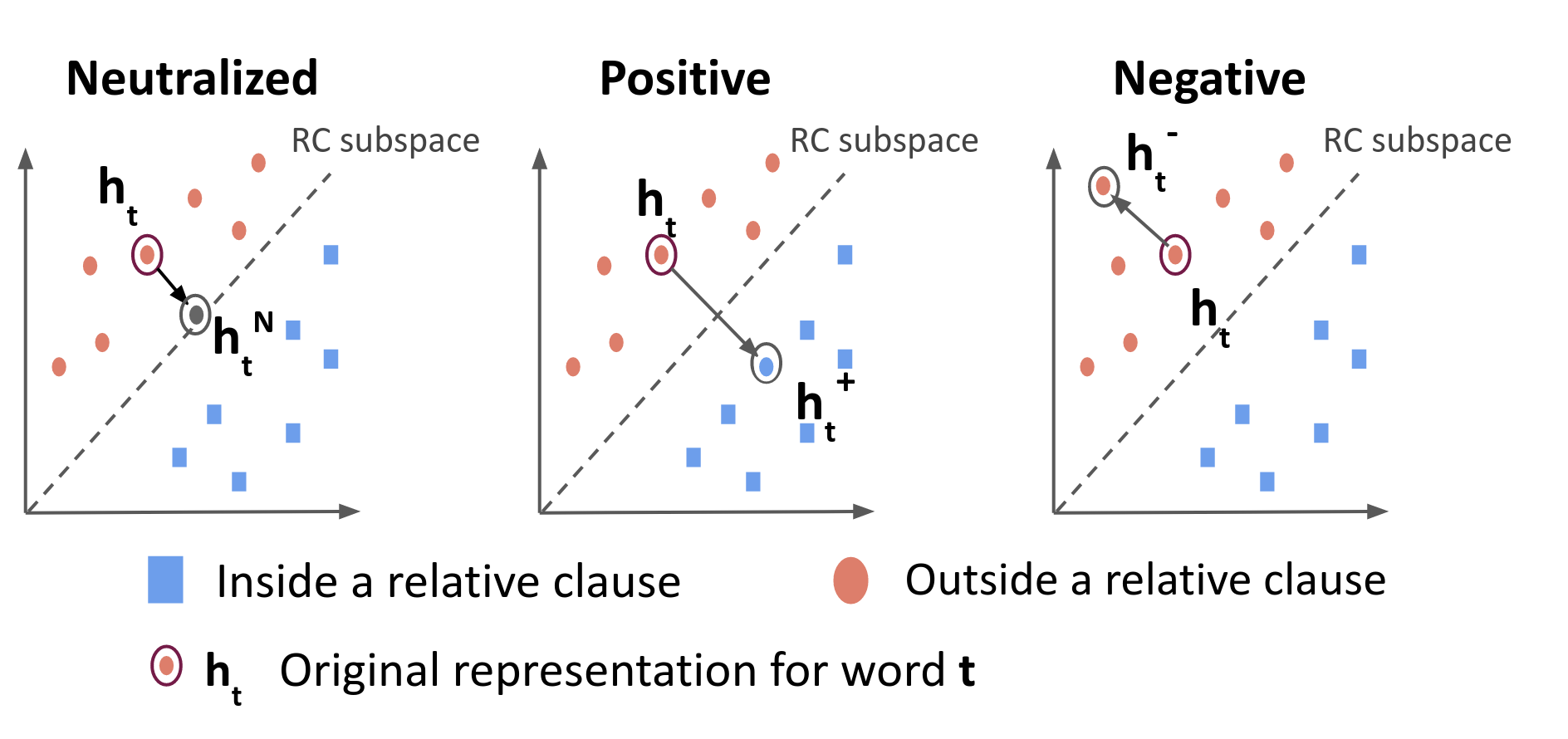}
\caption{Generating counterfactual representations. 
A representation $\vec{h}^t$ of a word outside of an RC is transformed to create counterfactual mirror images $\vec{h^-_t}$,$\vec{h^+_t}$  with respect to an empirical RC subspace. The RC subspace here is a 1-dimensional line for illustrative purposes; in practice we use an 8-dimensional subspace.}
\label{fig:counterfactuals}
\end{figure}

\section{\method: Generating Counterfactual Representations}
\label{sec:counterfactual}

The goal of \method\ is to generate, based on a model's contextual representations of a set of words, a set of counterfactual 
representations that modify the encoding of a feature $F$ while leaving all other aspects of the representations intact.\footnote{We aim to propose a concrete instantiation that \textit{approximates} the counterfactual.} 
If swapping these counterfactual representations for the model's original representations changes the model's probability distribution over predicted words in a way that aligns with the feature's linguistic functions, we say that the model uses $F$ for word prediction in a manner that is consistent with the grammar of English. 

For our case study, we use a feature with two possible values: `$+$' if the word is inside an RC and `$-$' if it is not.\footnote{In the experiments below, we will only apply this procedure to sets of representations of words that are all in a particular \emph{type} of RC (for example, Object RCs). We do, however, test whether the representations of RC boundary generalize across RC types; see Prediction~3 in \S\ref{src:predictions}.} We generate two counterfactual representations: $h_t^+$, which encodes that the word $t$ is inside an RC---regardless of its actual syntactic position in the sentence---and $h_t^-$, which is similar to $h_t^+$ in all respects except it encodes that $t$ is not in an RC. Our method allows us to generate $h_t^+$ and $h_t^-$ irrespective of the feature value encoded in the original representation $h_t$. If the model uses this feature appropriately, we expect $h_t^+$ and $h_t^-$ to lead to different predictions in contexts where correct word predictions depend on determining whether or not the word is inside an RC.

\paragraph{Row-space and Null-space}
Recall that INLP defines a feature space $R$ where the property of interest is encoded, and a complement subspace $N$ where it is not.

We can project any word representation $\vec{h_t}$ to the feature subspace (here, the RC subspace) or to the null space, resulting in the vectors $\vec{h_t^R}$ and $\vec{h_t^N}$, respectively: $\vec{h_t^R}$ maintains the information needed to predict $F$ from $\vec{h_t}$, while  $\vec{h_t^N}$ maintains all information which is \emph{not} relevant for predicting $F$.  
INLP can be used to generate ``amnesic counterfactuals" \citep{amnesic-probing}, which do not encode a given property, even if the original representation did encode that property. In the next paragraph we propose a way to use this algorithm to \textit{manipulate} the value of the feature, rather than remove it.

\paragraph{Generating Counterfactual Representations} We obtain the counterfactual representations $\vec{h_t^+}$ and $\vec{h_t^-}$ as follows. As we discussed in \Cref{sec:inlp}, INLP identifies planes---one for each direction (row) in $\matr{W}$---each of which linearly divides the word-representation space into two parts: words that are in an RC and words that are not. From the representation $\vec{h_t}$ of a word $t$ that is not in an RC, we can generate $\vec{h_t^+}$ by pushing $\vec{h_t}$ across the separating plane towards the representations of words that are inside an RC . Similarly, we can generate $\vec{h_t^-}$ by moving $\vec{h_t}$ \textit{further away} from that plane (see Figure~\ref{fig:counterfactuals}).\footnote{For a word $t$ that is \textit{inside} an RC, the reverse computations would be required: to generate $h_t^+$ we would move $\vec{h_t}$ further away from the separating plane, whereas to generate $h_t^-$ we would move $\vec{h_t}$ across the separating plane.}

How do we move the representation of a word away from or towards the separating planes? Recall that the feature subspace $R$ and the nullspace $N$ are orthogonal complements, and consequently any vector $\vec{v} \in \R^n$ can be represented as a sum of its projections on $R$ and $N$. Further, by definition, the vector's projection on $R$ is the sum of its projections on the RC directions $\vec{w} \in \matr{W}$. Thus, we can decompose $\vec{h_t}$ as follows, where $\vec{h_t^w}$ is the orthogonal projection of $\vec{h_t}$ on direction $\vec{w}$:

\begin{equation}
    \vec{h_t} = \vec{h_t^N} + \vec{h_t^R} = \vec{h_t^N}  + \sum_{\vec{w} \in W}{\vec{h_t^w}}
    \label{decomposition-equation}
\end{equation}

For any word $t$, we expect a positive counterfactual $\vec{h_t^{+}}$ to be classified as being \emph{inside} an RC, with high confidence, according to all \textit{original} RC directions $w \in \matr{W}$ --- that is, $\forall\ w \in \matr{W}, w^T\vec{h_t^{+}}>0 $. Conversely, we expect a negative counterfactual to be classified as \emph{not} being in an RC, i.e., $\forall\ w \in \matr{W}, w^T\vec{h_t^{-}} < 0$.

To enforce these desiderata, we create positive and negative counterfactuals as follows, where $\textit{SIGN}(x) = 1$ if $x \geq 0$ and $0$ otherwise, and $\alpha$ is a positive scalar hyperparameter that enhances or dampens the effect of the intervention.

\begin{align}
    \vec{h_t^{-}} &= \vec{h_t^{N}} + \alpha \sum_{\vec{w} \in \matr{W}}(-1)^{\textit{SIGN}(w^T\vec{h_t})} \vec{h_t^w} \label{eq:counterfactuals-neg}\\  
    \vec{h_t^{+}} &= \vec{h_t^{N}} + \alpha \sum_{\vec{w} \in \matr{W}}(-1)^{1 - \textit{SIGN}(w^T\vec{h_t})} \vec{h_t^w} \label{eq:counterfactuals-pos}
\end{align}

In both cases, we \emph{subtract} a direction $\vec{h_t^w}$, flipping its sign, if the sign constraints are violated, that is, if $w^T\vec{h_t} > 0$ for $\vec{h_t^{-}}$ and if $w^T\vec{h_t} < 0$ for $\vec{h_t^{+}}$. Geometrically, flipping the sign of a direction $\vec{h_t^w}$ in Equations~\ref{eq:counterfactuals-neg} and~\ref{eq:counterfactuals-pos} is equivalent to taking a \textit{mirror image} with respect to a direction $w$ (Figure~\ref{fig:counterfactuals}). This enforces the sign constraints: all classifiers $w$ predict the negative or positive class, respectively (see Appendix~\S\ref{app:proofs} for a formal proof).


\section{Experimental Setup}
Our overall goal is to assess the causal effect of RC boundary representations on our models' agreement behavior when subject-verb dependencies span an RC (that is, where an RC intervenes between the head of the subject and the corresponding verb). We test whether we can modify the representation of the masked verb outside the RC such that, compared to the original representations, the model assigns higher probability to either the correct form (after negative intervention) or to the incorrect one (after positive intervention). We first describe the models we use (\S\ref{sec:models}), then the dataset we use to obtain RC subspaces and generate counterfactual representations (\S\ref{sec:methods-counterfactual}), and finally the dataset we use to measure the models' agreement prediction accuracy in sentences containing RCs, before and after the counterfactual intervention (\S\ref{sec:methods-agreement}).

\subsection{Models}
\label{sec:models}
We use BERT-base (12 layers,768 hidden units) and BERT-large (24 layers, 1024 hidden units) \cite{bert}, as well as the smaller BERT models released by \citet{DBLP:journals/corr/abs-1908-08962}: BERT-medium (8 layers $\times$ 512 hidden units), BERT-small (4 layers, 512 hidden units), BERT-mini (4 $\times$ 256), and BERT-tiny (2 $\times$ 128). In all experiments, we intervene on a single layer at a time, and continue the forward pass of the original model through the following layers.

\subsection{Generating Counterfactual Representations}
\label{sec:methods-counterfactual}

\paragraph{Datasets}

To create the training data for the INLP classifiers, we used the templates of \citet{prasad-etal-2019-using} to generate five lexically matched sets of semantically plausible sentences, one for each type of RC outlined in Table~\ref{tab:ex}, as well as two additional sets of sentences without RCs; these included sentences with nearly the same word order and lexical content as the sentences in the other sets. Each set contained $4800$ sentences.
All verbs in the training sentences were in the past tense; this ensured that the subspaces we identified did not contain information about overt number agreement, making it unlikely that \method\ will alter agreement-related information that does not concern RCs.

\paragraph{Identifying and Altering RC Subspaces}
To identify RC subspaces, we used INLP with SVM classifiers as implemented in scikit-learn. We identified different subspaces for each of the five types of RCs listed in Table~\ref{tab:ex}. For example, in \ref{ex:src2}, the bolded words were considered to be in the RC.
\ex. \label{ex:src2} My cousin \textbf{that liked the book} hated movies.

For the negative examples, we took representations of words outside of the RC, either from outside the bolded region of the same sentence, or from inside or outside the bolded region of the coordination control sentence.

\ex. \label{ex:scont} My cousin \textbf{liked the book} and hated movies. 

We selected the negative examples in this manner for two reasons: first, to ensure that the same word served as a positive example in some context and as a negative example in others (e.g., \textit{book} in \ref{ex:src2} and \ref{ex:scont}); and second, to ensure that the same RC sentence included both positive and negative examples (e.g., \textit{book} and \textit{cousin} in \ref{ex:src2}).

\paragraph{Hyperparameters} INLP has a hyperparameter $m$ which sets the dimensionality of the RC subspace; this parameter trades off exhaustivity against selectivity.\footnote{In particular, running INLP for 768 iterations---the dimensionality of BERT representations---yields the original space, which is exhaustive but not useful in distilling RC information.} We set $m=8$;  In Appendix \S\ref{app:iters} we demonstrate that the trends we observe are not substantially affected by this parameter. 


\method\ has an hyperparameter, $\alpha$, that determines the magnitude of the counterfactual intervention (\S\ref{sec:methods-counterfactual}). We use $\alpha=4$; In Appendix \S\ref{app:alpha} we show that the trends we observe are similar for other values of $\alpha$.



\subsection{Measuring the Effect of the  Intervention on Agreement Accuracy}
\label{sec:methods-agreement}

\paragraph{Dataset}

We measure the models' agreement prediction accuracy using 
a subset of the \citet{marvin-linzen-2018-targeted} dataset in which the subject is modified by an RC. The noun inside the RC either matched \ref{ex:orc_agreement_match} or mismatched \ref{ex:orc_agreement_mismatch} the subject of the matrix clause in number:

\ex. \label{ex:orc_agreement_match} The skater \textbf{that the officer loves} is/are happy.

\ex. \label{ex:orc_agreement_mismatch} The skater \textbf{that the officers love} is/are happy.

The \citeauthor{marvin-linzen-2018-targeted} dataset contains sentences where the intervening RC is either a subject RC or a (reduced or unreduced) object RC. We augmented this dataset with lexically matched sentences with (reduced or unreduced) passive RC interveners, using attribute varying grammars \citep{mueller2020crosslinguistic}. Finally, we only considered sentences with copular main verbs (\textit{is} and \textit{are}) to ensure that both the singular and plural forms of the verb are highly frequent. We used 1750 sentences per construction.

\begin{figure*}[t]

    \centering
    \begin{subfigure}[t]{0.48\textwidth}
          \captionsetup{width=.9\textwidth}
            \includegraphics[width=\linewidth]{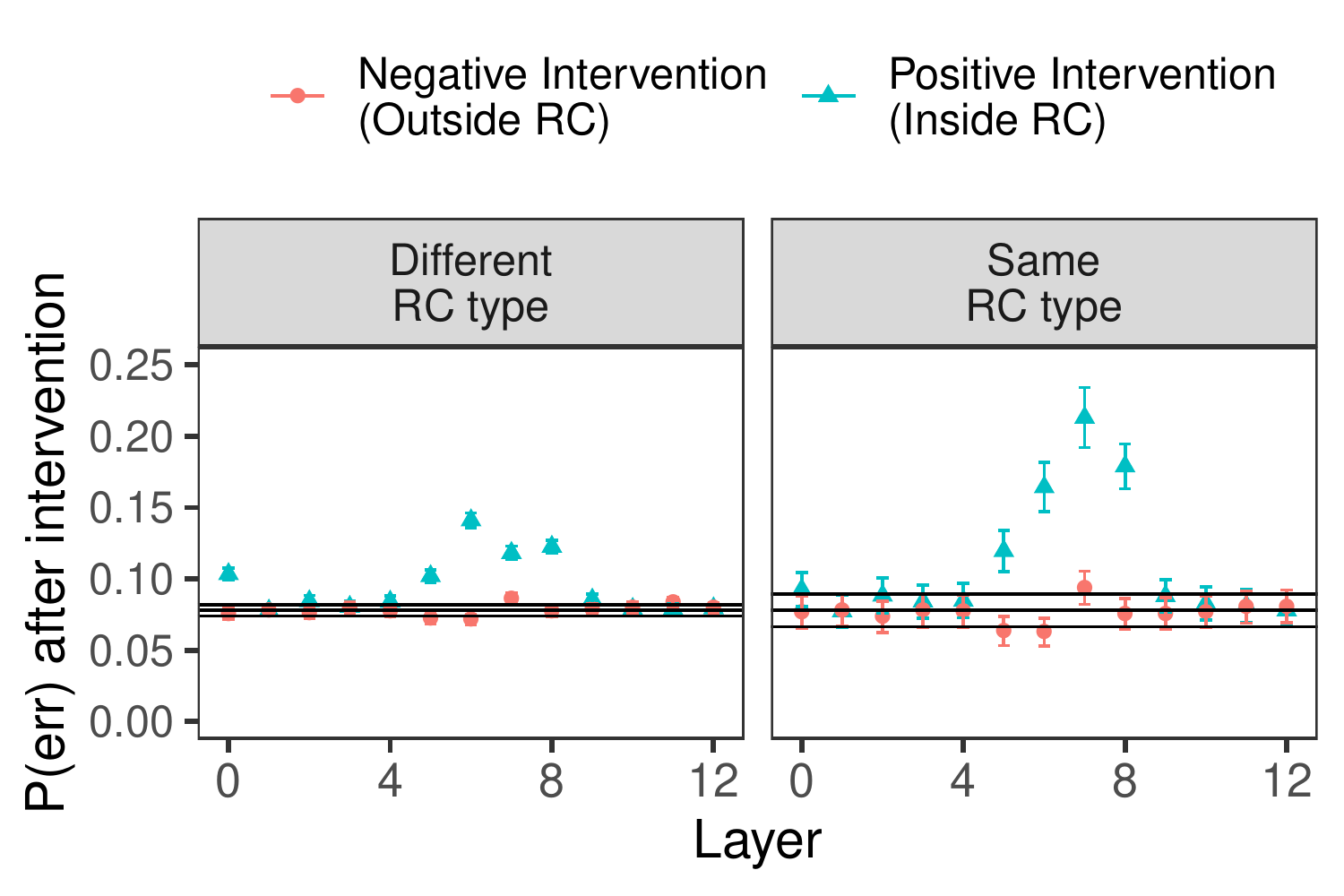}
            \caption{RC sentences with attractors. In the right panel, the test sentence included an RC of the type used to generate the counterfactual representations; 
            in the left panel, counterfactual representations were generated based on sentences with different RC types from those in the agreement test sentences. 
           } \label{fig:bert_base_perr_same_vs_diff}
    \end{subfigure}
    \begin{subfigure}[t]{0.48\textwidth}
            \captionsetup{width=.9\textwidth}
            \includegraphics[width=\linewidth]{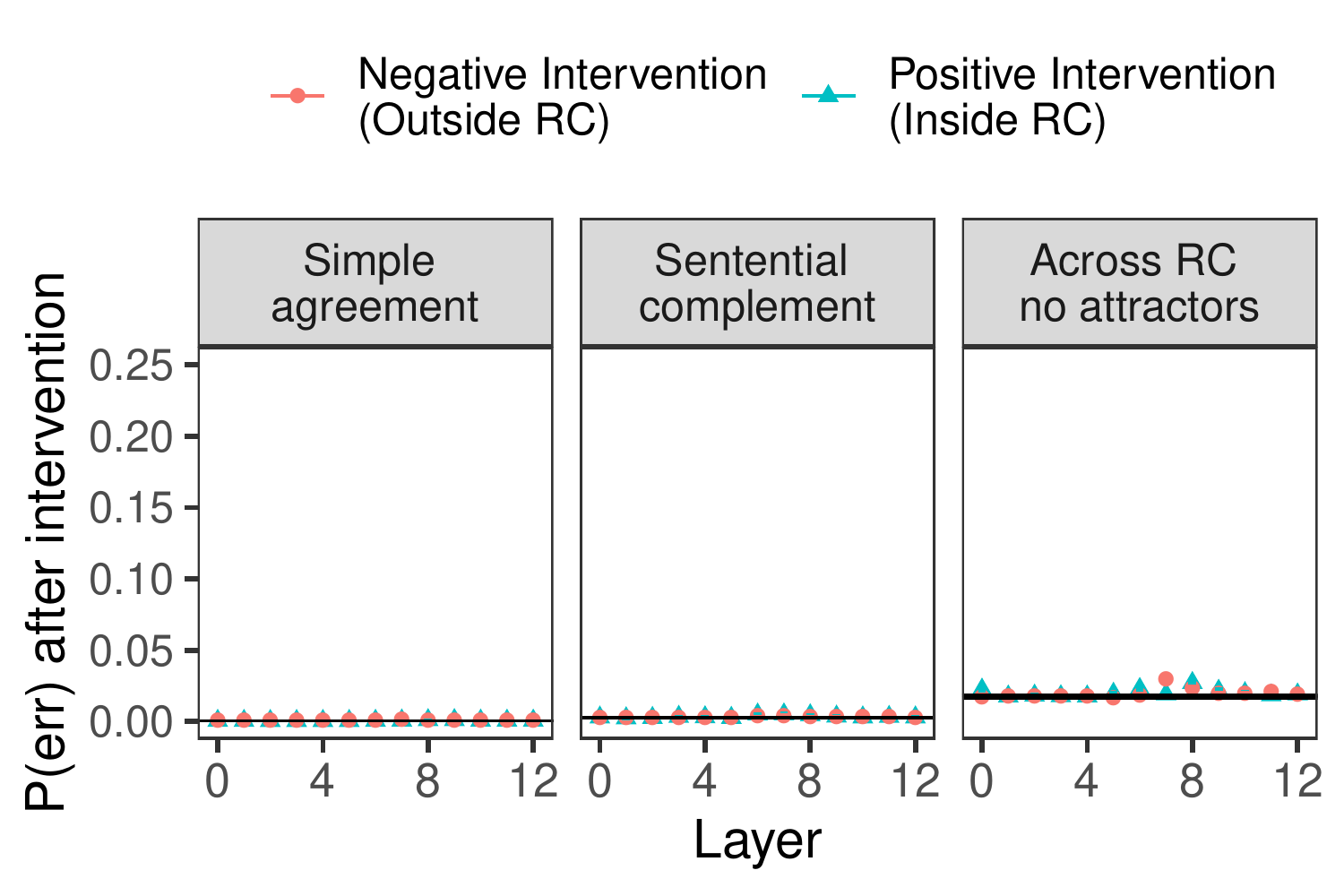}
            \caption{Sentences without RCs and sentences with an RC but without attractors.} \label{fig:bert_base_perr_nonattractorRC}
    \end{subfigure}
    \begin{subfigure}[t]{0.48\textwidth}
              \captionsetup{width=.9\textwidth}
        \includegraphics[width=\linewidth]{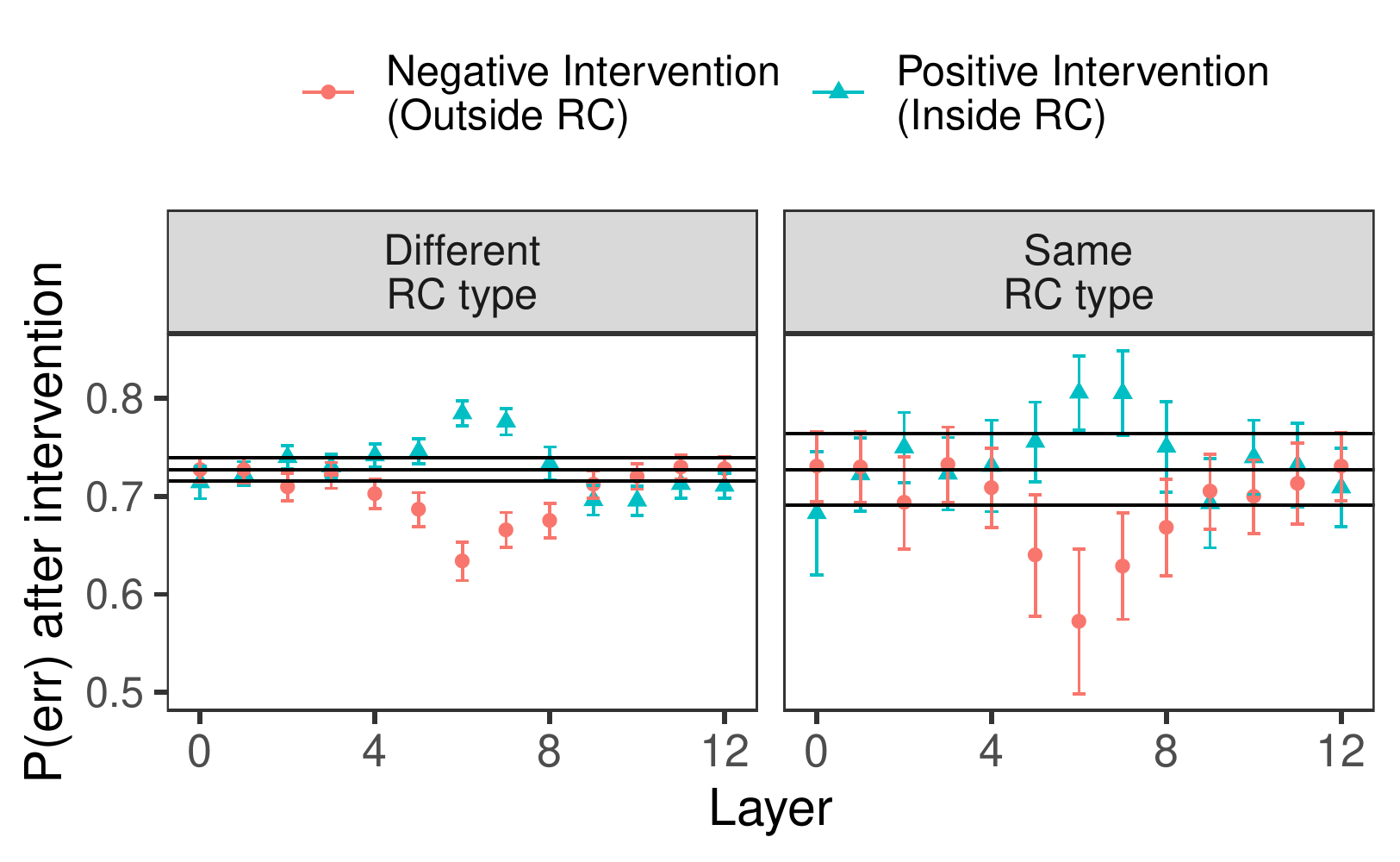}
        \caption{Sentences where before the intervention the model assigned a higher probability to the ungrammatical than the grammatical verb. Note the y-axis differs from other plots (reflecting the higher original error probability).}
        \label{fig:only_wrong_perr}
    \end{subfigure}
    \begin{subfigure}[t]{0.48\textwidth}
    \captionsetup{width=.9\textwidth}
        \includegraphics[width=\linewidth]{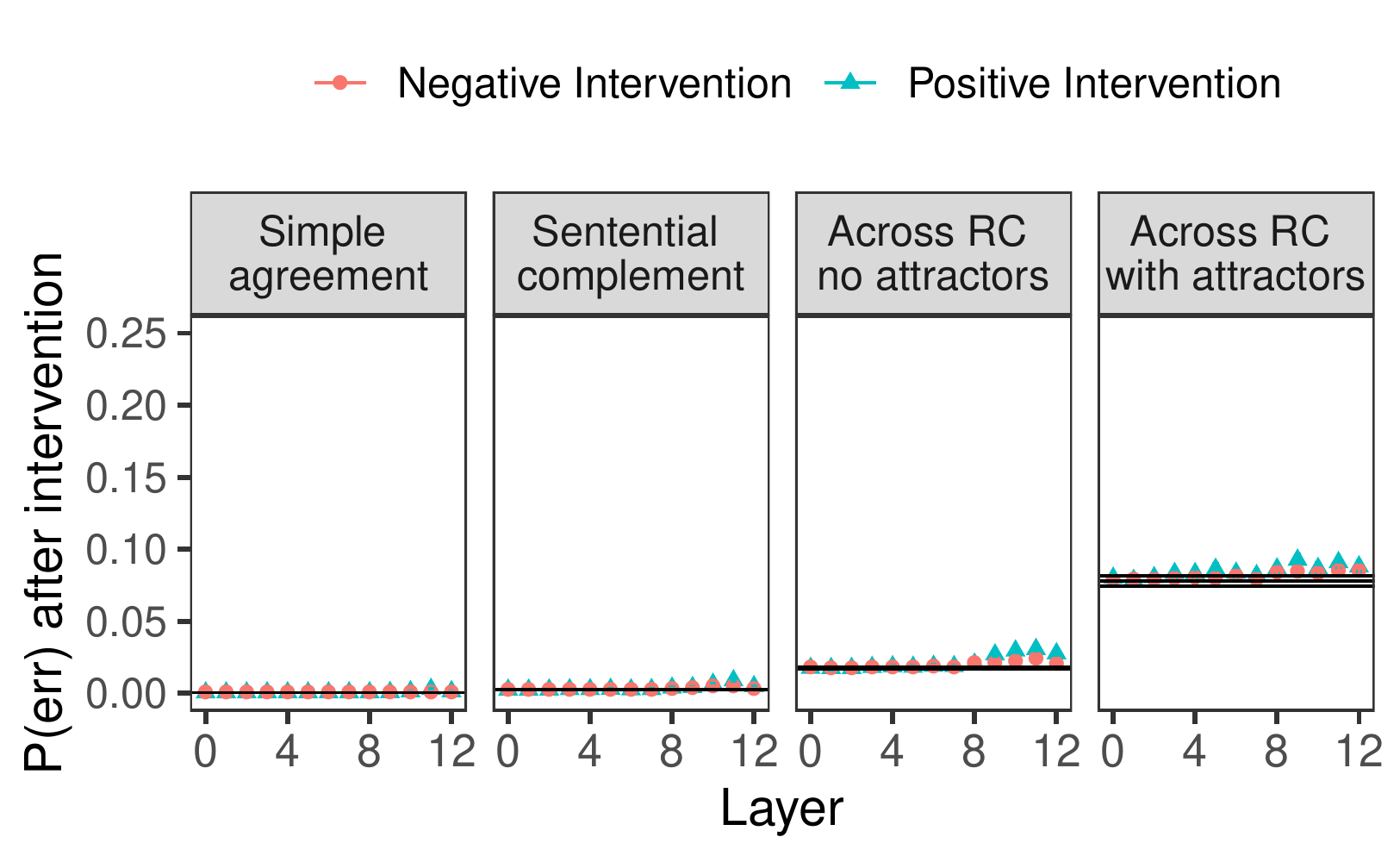}
        \caption{Intervention from counterfactual representations generated from 10 random subspaces. }
        \label{fig:perr_random_subspace}
    \end{subfigure}
    \caption{\label{fig:agreement_bert_base}Change in 
           probability of error with negative and positive counterfactual BERT-base representations (red circle and cyan triangle respectively). Horizontal lines indicate probability of error with the original representations without any intervention: the middle line is the mean accuracy across all items prior to intervention and the upper and lower lines indicate accuracy two standard errors away from the mean accuracy. Error bars reflect two standard errors from the mean probability of error after intervention. } 
\end{figure*} 

\paragraph{Computing Agreement Accuracy}

We performed masked language modeling (MLM) on the dataset described earlier in this section. In each sentence, we masked the copula, started the forward pass, performed the intervention on the representation of the masked copula in the layer of interest, and continued with the forward pass to obtain BERT's distribution over the vocabulary for the masked token. We repeated this process for each layer separately. We then computed the probability of error, normalized within the two copulas \emph{is} and \emph{are} (\citealt{arehalli20}):

\begin{equation}
P(\textit{Err}) = \dfrac{P(\textit{Verb}_{\textit{Incorrect}})}{P(\textit{Verb}_{\textit{Incorrect}}) +P(\textit{Verb}_{\textit{Correct}})}
\end{equation}

In Appendix~\S\ref{app:accuracy}, we present results where the metric of success is accuracy, that is, the percentage of cases where the model assigned a higher probability to the verb with the correct number \cite{marvin-linzen-2018-targeted}. These results are qualitatively similar.

\section{Predictions}
\label{src:predictions}

As discussed earlier, a system that computed agreement in accordance with the grammar of English would determine the number of the masked verb in a sentence like \ref{ex:src_agreement2} based on the number of \textit{officers}, because both \textit{officers} and the \mask\ token are outside the RC; the number of the RC-internal noun \textit{skater} should be ignored. 
\ex. \label{ex:src_agreement2} The officers \textbf{that love the skater} \mask\ nice. 

We can derive the following predictions for applying \method\ to a system that follows this strategy:

\paragraph{Prediction 1: Impact on Error Probability in RC Sentences with Attractors.} 
In RC sentences where the main clause subject differs in number from the RC subject, error probability will be higher with the counterfactual $h_{\textit{MASK}}^+$, which encodes (incorrectly) that \mask\ is \textit{inside} the RC, than with the original representation $h_{\textit{MASK}}$. Conversely, error probability will be lower with $h_{\textit{MASK}}^-$, which encodes (correctly) that \mask\ is \textit{outside} the RC, than with the original $h_{\textit{MASK}}$. 

\paragraph{Prediction 2: No Impact on Other Sentences.}
We do not expect a difference in error probability between the original and counterfactual representations in all other sentences. This should be the case for sentences with RCs where the nouns inside and outside the RC match in number, as in~\ref{ex:src_agreement2.2}:

\ex. \label{ex:src_agreement2.2} The officers \textbf{that love the skaters} \mask\  nice. 

Since both \textit{officers} and \textit{skaters} are plural, most plausible agreement prediction strategies would make the same predictions regardless of whether \mask\ is analyzed as being inside the RC or outside it. Consequently, intervening on the encoding of RC boundaries is not expected to systematically change the model's predictions.

Likewise, since the interventions are designed to modulate the encoding of RC-related properties, we do not expect the interventions to impact number prediction in sentences without RCs such as \ref{ex:simple_agreement} and \ref{ex:sent_comp}:\footnote{If models encoded boundaries of all embedded clauses similarly we would expect a change in prediction for \ref{ex:sent_comp}.}

\ex. \label{ex:simple_agreement} The officer \mask\ nice.

\ex. \label{ex:sent_comp} The bankers knew the officer \mask\ nice.

\paragraph{Prediction 3: Generalization Across RC Types.}
If RC boundaries are represented in an abstract way that is shared across different RC types, then the counterfactual representations will affect error probability in the same way regardless of whether the counterfactual representations were generated from subspaces estimated from sentence with the same RC type as the target sentences, or from sentences with different RC types.



\section{Results}
\label{intervention_results}

\paragraph{Counterfactual Intervention in the Middle Layers of BERT-base Modulates Agreement Error Rate in RC Sentences with Attractors.}

We begin by discussing experiments where subspaces were estimated from sentences with the same type of RC as the test sentences with agreement; we report results averaged across the five RC types.
Interventions using counterfactual representations generated from the middle layers of the BERT-base (5--8 out of~12) resulted in changes in the probability of error which partially aligned with Prediction~1 (Figure~\ref{fig:bert_base_perr_same_vs_diff}). In sentences with attractors, using the positive counterfactual $h_{\textit{MASK}}^+$ resulted in an increase in the probability of error (a maximum increase of 14 percentage points in layer~7). Conversely, using the negative counterfactual $h_{\textit{MASK}}^-$ generated from layers~5  and~6 resulted in a decrease in the probability of error. This decrease was much smaller (a maximum decrease of~2 percentage point in layer~6) and there was overlap in the error bars for the probability of error before and after intervention. 

It is likely that the smaller effects of the negative counterfactual intervention are due to the fact that accuracy before the intervention was very high (95\%) and the probability of error very low (8\%), leaving very little room for change: in most cases, the original representation already correctly encoded the verb is outside the RC. In a follow-up analysis, we only considered sentences in which the model originally assigned a higher probability to the ungrammatical than the grammatical form. In these examples the decrease in probability of error was larger (a maximum decrease of~16 percentage points in layer~6; see Figure~\ref{fig:only_wrong_perr}). 

While only RC interventions in the middle layers elicit the expected behavioral outcomes, probing accuracy for RC information was high for \text{all} layers (Appendix \S\ref{app:probing}), giving further evidence to the dissociation between correlational and causal methods: probing can identify aspects of the representations that do not affect the model's behavior.

\paragraph{Interventions on RC Boundary Representations Generalize Across RC Types, but not Further.} In line with Prediction 3, we observed a qualitatively similar pattern of change in error probabilities even when the counterfactuals were generated from subspaces estimated from a \textit{different} RC type than the RC in the agreement test sentences. The effects were smaller, however. This suggests that while BERT's representation of RC boundaries is partly shared across different RC types, 
there are also structure-specific RC boundary representations. 
The effect of the intervention also aligned with Prediction~2: in constructions where we do not expect RC boundaries to affect predictions---sentences without attractors and those without RCs---we did not observe significant changes in error probability (Figure~\ref{fig:bert_base_perr_nonattractorRC}). 

\paragraph{Intervention Based on Random Subspaces Does Not Produce Interpretable Results.}

To tease apart the effect of the RC-targeted intervention from intervening on \textit{any} subspace with the same dimensionality, we generated counterfactual representations from 10 random subspaces and repeated our analysis.\footnote{We generated a random subspace by sampling standard Gaussian vectors instead of the INLP matrix $\matr{W}$, and then employing the same procedure described in \S\ref{sec:counterfactual}.} While we observed very small changes in probability of error in some cases, the pattern of changes resulting from this intervention did not align with any of our predictions (see Figure~\ref{fig:perr_random_subspace}). This suggests that the change in probability of error that resulted from intervening with RC subspaces was not merely a by-product of intervening on a large enough subspace of BERT's original representation space.

\paragraph{Intervening on the Middle Layers of Other BERT Variants Yielded Qualitatively Similar Results.}
\begin{figure}
    \centering
    \includegraphics[width=0.45\textwidth]{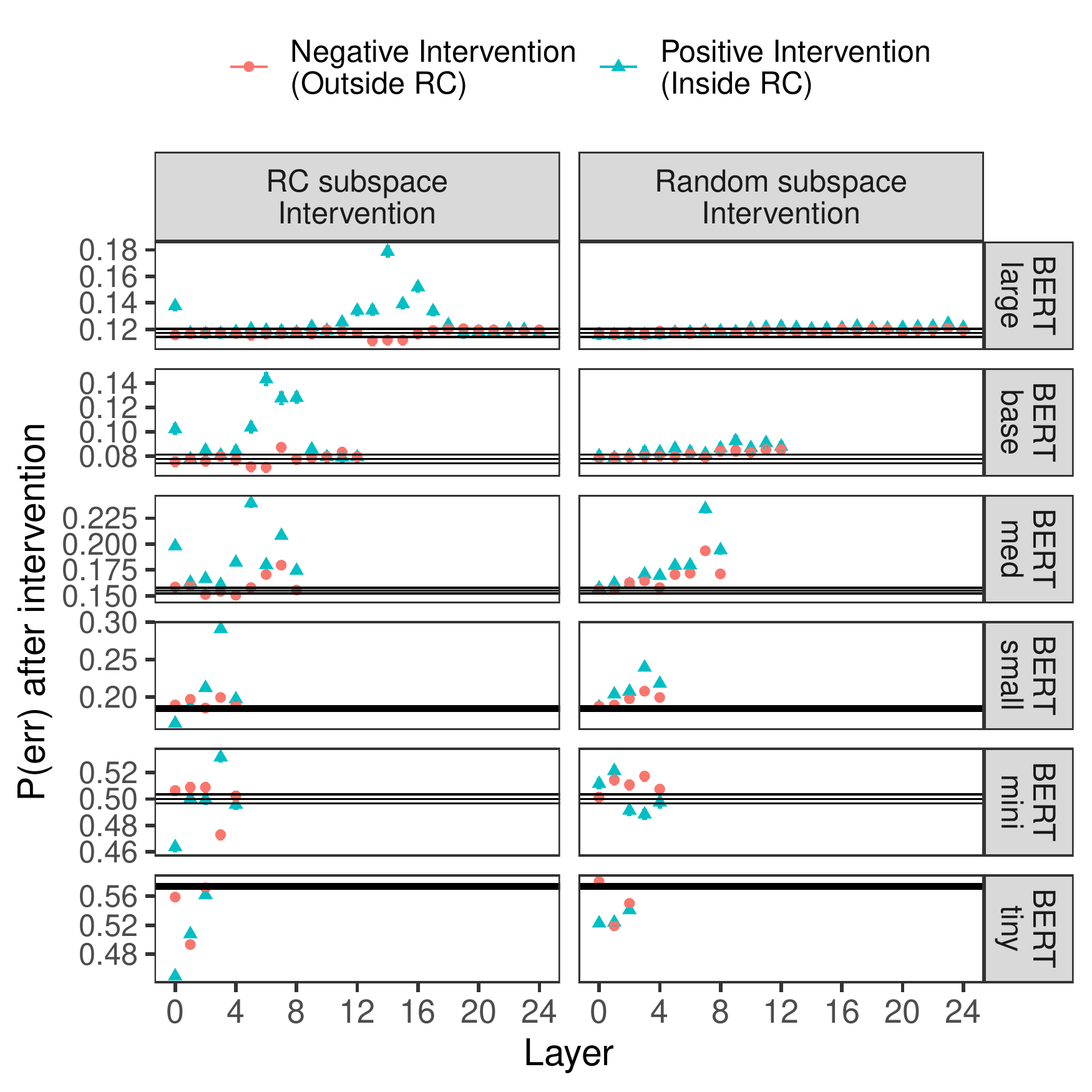}
    \caption{Change caused by counterfactual representations in agreement error probability across RCs with attractors for different BERT variants. Note that the baseline performance prior to intervention (marked by black horizontal lines) is different between models.}
    \label{fig:other-berts}
\end{figure}

We repeated the experiments on BERT-large and four smaller versions of BERT, trained on the same amount of data as the BERT-base model \cite{DBLP:journals/corr/abs-1908-08962}. As with BERT-base, intervening on the middle layers of BERT-large (12--17 out of~24) with the RC subspaces---but not the random subspaces---resulted in predicted changes in the probability of error. Compared to BERT-base, the smaller models showed a greater change in the probability of error as a result of intervention with counterfactuals generated from random subspaces. However, when the counterfactual representations were generated from particular layers---4 and~5 (out of~8) in BERT-medium, 3 (out of~4) in BERT-mini and 2 (out of~4) in BERT-small---the change in error probability aligned with Prediction~1 over and above the changes from intervening with random subspaces. In all of these layers, intervening with the positive but not the negative counterfactual resulted in an increase in the probability of error. No such layer was observed for BERT-tiny, which has only~2 layers (see Figure~\ref{fig:other-berts}). 

\section{Discussion}
We proposed an intervention-based method, \method, to test whether language models use the linguistic information encoded in their representations in a manner that is consistent with the grammar of the language they are trained on. For a given linguistic feature of interest, we generated counterfactual contextual word representations by manipulating the value of the feature in the original representations. Then, by replacing the original representations with these counterfactual variants, we characterized the change in word prediction behaviour. By comparing the resulting change in word prediction with hypotheses from linguistic theory about how specific values of the feature are expected to influence the probabilities over predicted words, we investigated whether the model \textit{uses} the feature as expected. 

As a case study, we applied this method to study whether altering the information encoded about RC boundaries in the contextual representations of masked verbs in different BERT variants influences the verb's number inflection in a manner that is consistent with the grammar of English. We found that while all layers of the BERT variants encoded information about RC boundaries, only the information in the middle layers influenced the masked verb's number inflection as predicted by English grammar. We also found that in BERT-base, counterfactual representations based on subspaces that were learned from sentences with one type of RC influenced the number inflection of the masked verb in sentences with other types of RCs; this suggests that the model encodes information about RC boundaries in an abstract manner that generalizes across the different RC types. 

\paragraph{Caveat: Linear Analysis of a Non-linear Network} \method\ interventions are based on concept subspaces identified using linear classifiers, but most neural networks components, including BERT layers, are non-linear. It is possible, then, that subsequent non-linear layers transform the counterfactual representation in a way that is not amenable to analysis using our methods. As such, while we can conclude from a positive result that the feature in question causally affects the model's behavior, negative results should be interpreted more cautiously.

\paragraph{Future Work} Future work can apply this method to test linguistically motivated hypotheses about a wide range of structures and tasks. For example, linguistic theory predicts that information about semantic roles (like agent and patient) is crucial for tasks such as natural language inference (NLI) that require reasoning about sentence meaning. To test if NLI models use semantic roles as predicted by linguistic theory, we can use \method\ to replace the original representations with counterfactual representations where the patient is encoded as the agent (and vice versa), and measure the change in performance on NLI, especially on challenge sets such as HANS \cite{mccoy2019right} that evaluate sensitivity to these properties.

\section{Related Work}
\label{related-work}


\paragraph{Probing and Causal Analysis} Behavioral tests of neural models, such as the ability of the model to master agreement prediction  \citep{linzen2016assessing, gulordava2018LMagreement, goldberg2019bert_agreemenet}, have exposed both impressive capabilities and limitations. These paradigms focus on the model's output, and do not link the behavioral output with the information encoded in its representations. Conversely, probing \citep{diagnostic_adi, cram_vectors_conneau, hupkes2018visualisation}
does not reveal whether the property recovered by the probe affects the original model's prediction in any way \citep{hewitt2019control, tamkin2020investigating, abilasha_probing}. This has sparked interest in identifying the \emph{causal} factors that underlie the model's behavior \citep{DBLP:journals/corr/abs-2004-12265, DBLP:journals/corr/abs-2005-13407, DBLP:journals/corr/abs-2010-10907, DBLP:conf/iclr/KaushikHL20, DBLP:conf/naacl/SlobodkinCA21, DBLP:conf/naacl/PryzantCJVS21,finlayson-etal-2021-causal}.

\paragraph{Counterfactuals} The relation between counterfactual reasoning and causality is extensively discussed in social science and philosophy literature \cite{woodward2005making, DBLP:journals/corr/abs-1811-03163, DBLP:journals/ai/Miller19}. Attempts have been made to generate counterfactual examples \cite{DBLP:conf/emnlp/MaudslayGCT19, DBLP:conf/acl/ZmigrodMWC19, DBLP:journals/corr/abs-2012-13985, DBLP:conf/iclr/KaushikHL20, DBLP:journals/corr/abs-2103-13701} and recently to derive counterfactual representations \cite{DBLP:journals/corr/abs-2005-13407, amnesic-probing, DBLP:journals/corr/abs-2103-01378, DBLP:conf/emnlp/ShinSJKJM20, DBLP:journals/corr/abs-2105-14002}. Contrary to our approach, previous attempts to generate counterfactual representations were either limited to amnesic operations (i.e., focused on the \textit{removal} of information and not on \textit{modifying} the encoded information) or used gradient-based interventions, which are expressive and powerful, but less controllable. Our linear approach is guided by well-defined desiderata: we want all linear classifiers trained on the original representation to predict a specific class for the counterfactual representations, and we \textit{prove} that is the case in Appendix \S\ref{app:proofs}. 

\paragraph{Representations and Behavior} Previous work bridging the gap between representations and behavior includes \citet{GiulianelliHMHZ18}, who demonstrated that back-propagating an agreement probe into a language model induces behavioral changes and improve predictions. \citet{lakretz2019emergence} identified individual neurons that causally support agreement prediction. \citet{prasad-etal-2019-using} used similarity measures between different RC types extracted using behavioural methods to investigate the inner organization of information within the model.
Closest to our work is  \citet{amnesic-probing}, where the authors applied INLP to ``erase" certain distinctions from the representation, and then measured the effect of the intervention on language modeling. We extend INLP to generate flexible counterfactual representations (\S\ref{sec:counterfactual}) and use these to instantiate hypotheses about the linguistic factors that guide the model's behavior.


\section{Conclusions}

We proposed an intervention-based approach to study whether a model uses a particular linguistic feature as predicted by the grammar of the language it was trained on. To do so, we generated counterfactual representations in which the linguistic property under consideration was altered but all other aspects of the representation remained intact. Then, we replaced the original word  representation with the counterfactual one and characterized the change in behaviour. Applying this method to BERT, we found that the model uses information about RC boundaries that is encoded in its word representations when inflecting the number of masked verb in a manner consistent with the grammar of English. We conclude that \method\  is an effective tool for testing hypotheses about the function of the linguistic information encoded in the internal representations of neural LMs. 

\section*{Acknowledgements}

This work was supported by United States--Israel Binational Science Foundation award 2018284, and has received funding from the European Research Council (ERC) under the European Union's Horizon 2020 research and innovation programme, grant agreement No. 802774 (iEXTRACT). We thank Robert Frank for a fruitful discussion of an early version of this work, and Marius Mosbach, Hila Gonen and Yanai Elazar for their helpful comments.

\bibliography{emnlp2021}
\bibliographystyle{acl_natbib}

\clearpage
\newpage
\appendix

\section{Appendix}

\subsection{Correctness of the Counterfactual Generation}
\label{app:proofs}
In this appendix, we prove that the method presented in \S\ref{sec:counterfactual} is guaranteed to achieve its goal: the negative counterfactual $\vec{h_t^{-}}$ would be classified as belonging to the negative class, and the positive counterfactual $\vec{h_t^{+}}$ would be classified as belonging to the positive class, according to all the linear classifiers $w$ trained on the original representation. 

We base our derivation on the decomposition presented in \S\ref{sec:counterfactual}:

\begin{equation}
    \vec{h_t} = \vec{h_t^N} + \vec{h_t^R} = \vec{h_t^N}  + \sum_{\vec{w} \in W}{\vec{h_t^w}}
    \label{decomposition-equation}
\end{equation}

Where $N$ is the nullspace of the INLP matrix $\matr{W}$, $R$ is its rowspace, and $\vec{h_t^N}$ and $\vec{h_t^R}$ are the orthogonal projection of a representation $\vec{h_t}$ to those subspaces, respectively. 

We focus on the negative counterfactual; The proof for the positive counterfactual is similar. In the proceeding discussion,  $w_j \in \R^d$ is an arbitrary linear classifier trained on the $j$th iteration of INLP (one of the rows in the matrix $\matr{W}$). $w_j$ predicts a negative or positive class $\hat{y} \in \{0,1\}$ according to the decision rule $\hat{y} = SIGN(w_j^Th_t)$ \footnote{For simplicity, we define $SIGN(x) = 1$ if $x \geq 0$ else $0$. $0$ corresponds to the negative class.}. We denote by $\vec{h_t}^{w}$ the orthogonal projection of $\vec{h_t}$ on a direction $w$, given by $(\vec{h_t}^Tw) \vec{w}$.

\begin{claim}
For the negative counterfactual defined by $\vec{h_t^{-}} = \vec{h_t^{N}} + \alpha \sum_{i=0}^{m} {(-1)^{SIGN(w_i^T\vec{h_t})} } \vec{{h_t}^{w_i}}$, it holds that $\vec{h_t^{-}}$ would always be classified to the negative class: $w_j^Th_t^{-}<0$ for every $w_j$ in the original INLP matrix $\matr{W}$.
\end{claim}

\begin{proof}
\begin{align}
    w_j^T \vec{h_t^{-}} &= w_j^T(\vec{h_t^{N}} + \alpha \sum_{i=0}^m(-1)^{SIGN(w_i^T\vec{h_t})} \vec{{h_t}^w_i}) \label{first} \\
    &= w_j^T(\alpha \sum_{i=0}^m(-1)^{SIGN(w_i^T\vec{h_t})} \vec{{h_t}}^{w_i}) \label{second} \\
    &= \alpha {w_j}^T((-1)^{SIGN({w_j}^T\vec{h_t})} \vec{{h_t}^{w_{j}}}) \label{third}
\end{align}

Where the transition from \ref{first} to \ref{second} stems from $\vec{h_t^{N}}$ being in the nullsapce of $\matr{W}$, so $\forall w \in \matr{W}: w^T\vec{h_t^{N}}=0$; and the transition from \ref{second} to \ref{third} stems from the mutual orthogonality of the INLP classifiers (proved in \citet{inlp}): since $\forall\ j \neq i, w_i^Tw_j=0$, it holds that $w_j^T\vec{h_t^{w_{i}}}=w_j^T( (\vec{h_t}^Tw_i) w_i  ) =  (\vec{h_t}^Tw_i)w_j^Tw_i=0$.

Now, we consider two cases.
\begin{itemize}
    \item \textbf{Case 1:} $w_j^T\vec{h_t}>0$, that is, the classifier predicted the positive class on the \textit{original} representation. Then, by \ref{third}, 
    
    \begin{align}
        w_j^T\vec{h_t^{-}} &= \alpha w_j^T((-1)^{SIGN(w_j^T\vec{h_t})} \vec{h_t^{w_{j}}}) \\
        &= \alpha w_j^T (-1)h_t^{w_j} \\
        & = -\alpha w_j^Th_t^{w_j}
    \end{align}
    
    Since $\alpha$ is a positive scalar and by assumption $w_j^T\vec{h_t}>0$, it holds that $w_j^T\vec{h_y^{-}}<0$.

        \item \textbf{Case 2:} $Sw_j^T\vec{h_t}<0$, that is, the classifier predicted the negative class on the \textit{original} representation. Then, by \ref{third}, 
    
    \begin{align}
        w_j^T\vec{h_t^{-}} &= \alpha w_j^T((-1)^{SIGN(w_j^T\vec{h_t})} \vec{h_t^{w_{j}}}) \\
        &= \alpha w_j^T h_t^{w_j} \\
        & = \alpha w_j^Th_t^{w_j}
    \end{align}
    
    Since $\alpha$ is a positive scalar and by assumption $w_j^T\vec{h_t}<0$, it holds that $w_j^T\vec{h_y^{-}}<0$.
\end{itemize}

We have proved that regardless of the originally predicted label, all INLP classifiers would predict the negative class on the negative counterfactual, which concludes the proof.
\end{proof}

\subsection{Probing Accuracy}
\label{app:probing}

\begin{figure}[h]
    \centering
    \includegraphics[width=\columnwidth]{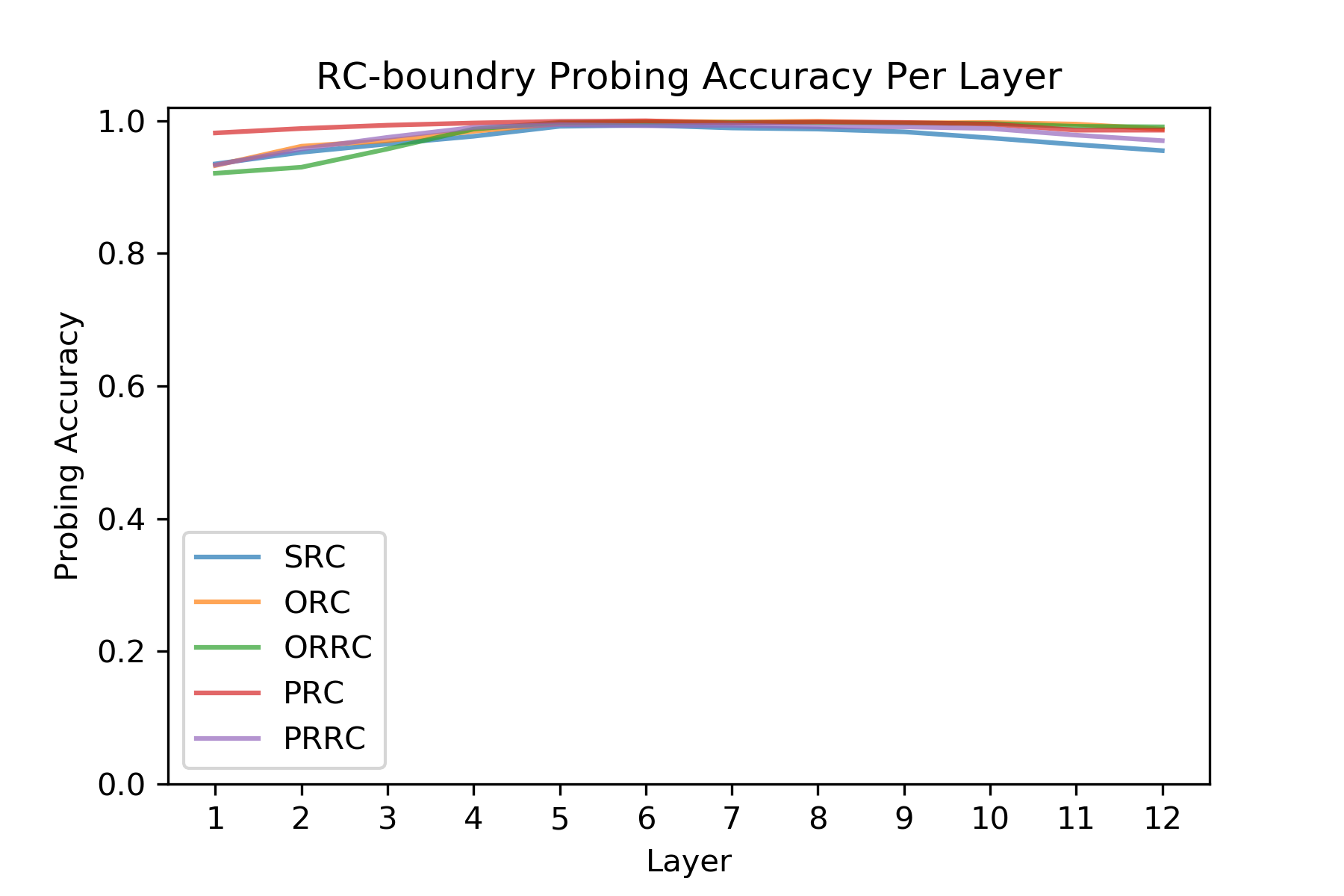}
    \caption{Probing accuracy for the presence of words within or outside of RCs,  vs. BERT-base layers, for all the different RC types in our experiments. }
    \label{fig:probing}
\end{figure}

In this appendix, we provide probing results for the task on which we run INLP: detecting whether representation was taken over a word inside or outside of an RC. As INLP iteratively trains linear probes, this accuracy is equivalent to the accuracy of the first INLP classifier. In all contextualized layers, we observe probing accuracy of over 90\% for all RC types (Figure \ref{fig:probing}). This contrasts with the intervention results in \S\ref{intervention_results}. While it is possible to linearly decode the RC boundary in \textit{all} layers, only in the middle layers do we find that this concept \textit{causally} influences the model's behavior. In other words, good probing performance does not indicate main-task relevancy.  

\subsection{Influence of the Dimensionality of the RC Subspace}
\label{app:iters}

\begin{figure}[h]
\begin{subfigure}[b]{0.41\textwidth}
   \includegraphics[width=0.95\linewidth]{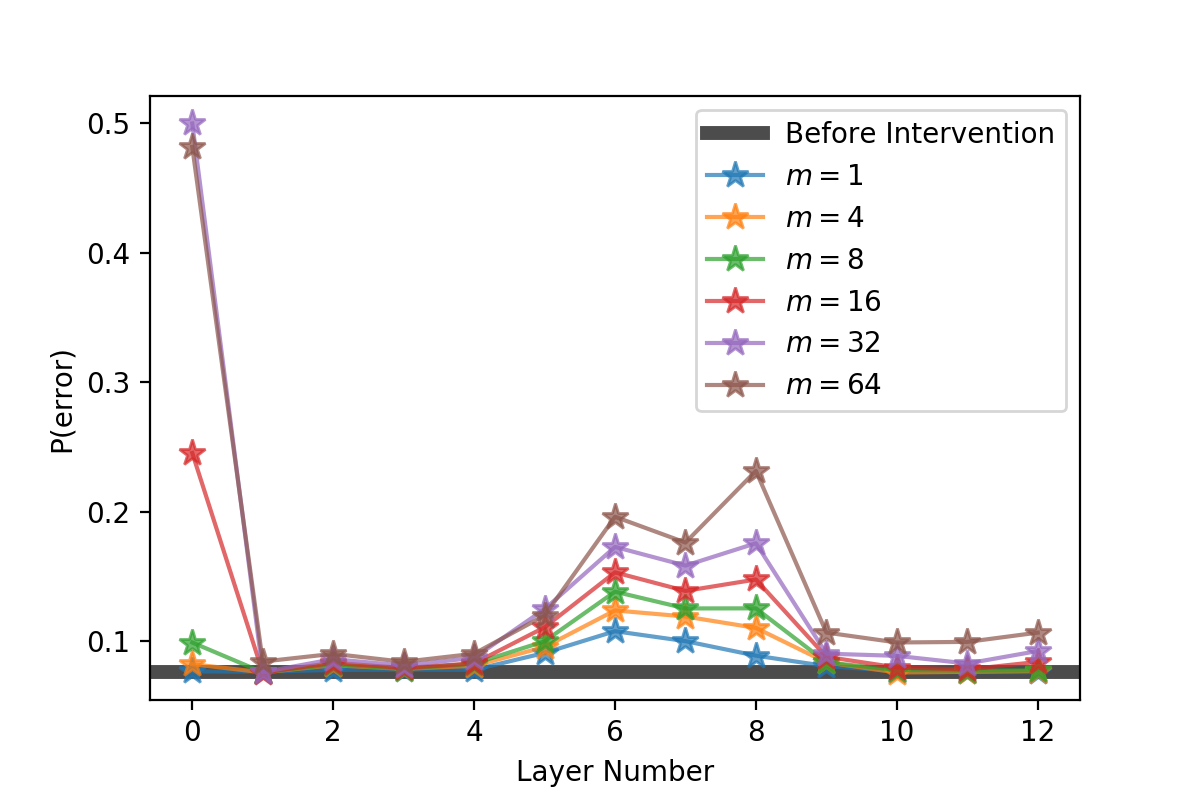}
   \caption{\textbf{Positive} intervention results on sentences with agreement across RC with attractors.}
   \label{fig:iters-agreement} 
\end{subfigure}

\begin{subfigure}[b]{0.41\textwidth}
   \includegraphics[width=0.95\linewidth]{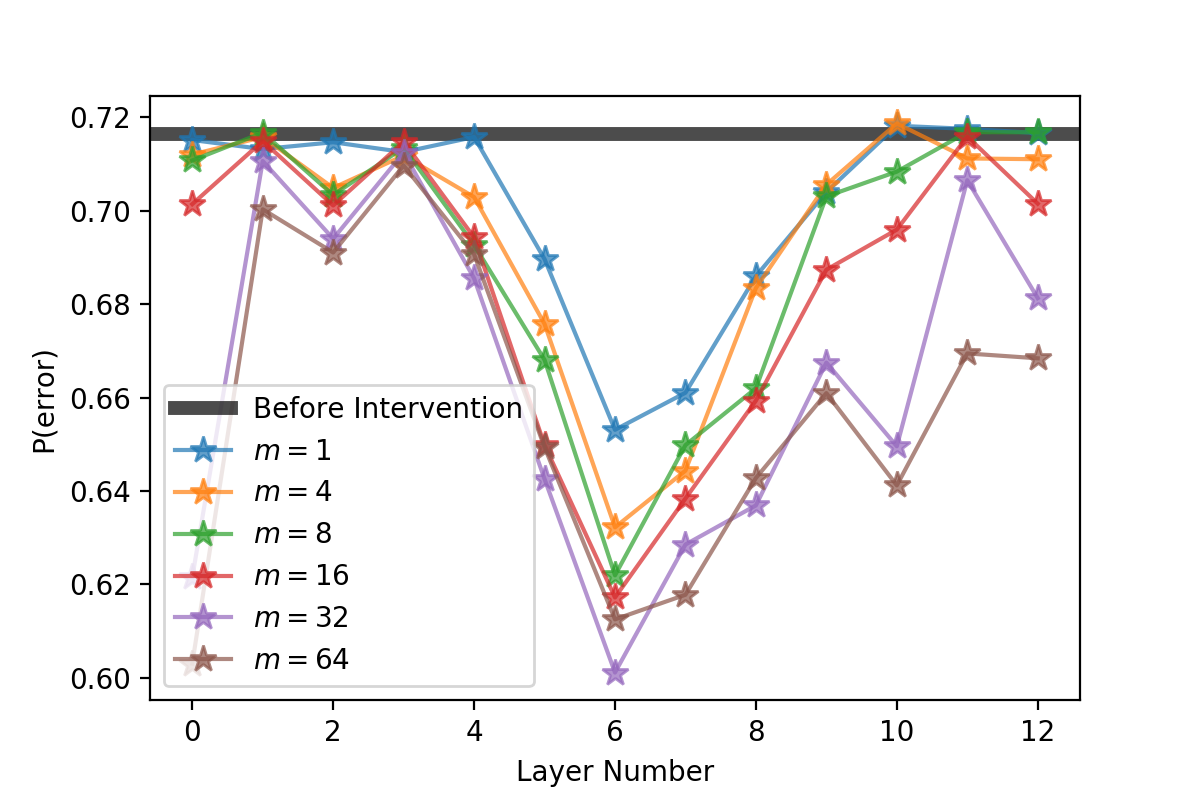}
   \caption{\textbf{Negative} intervention results on sentences with agreement across RC on which the model originally predicted incorrectly.}
   \label{fig:iters-error}
  
\end{subfigure}

\begin{subfigure}[b]{0.41\textwidth}
   \includegraphics[width=0.95\linewidth]{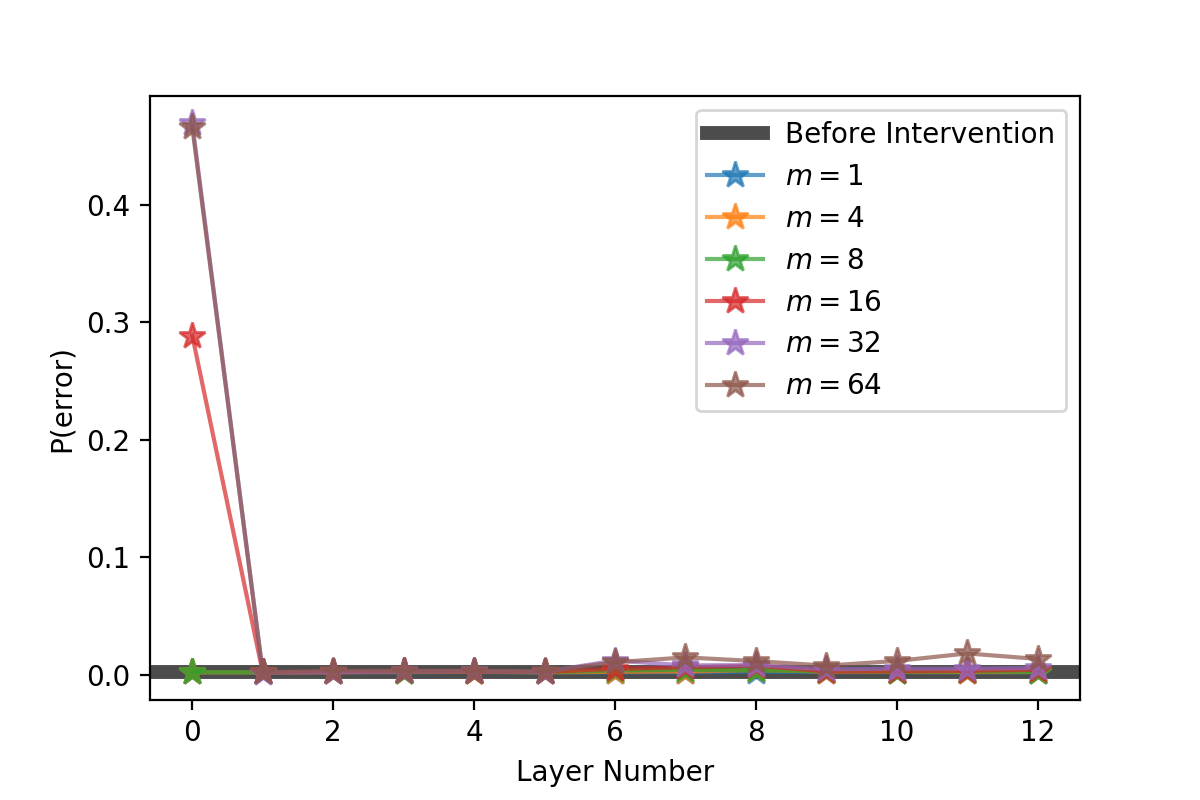}
   \caption{\textbf{Positive} intervention results on sentences with simple agreement and sentences with sentential complements.}
   \label{fig:iters-no-agreement}
  
\end{subfigure}

\end{figure}

In this appendix, we analyze the influence of the dimensionality $m$ of the RC subspace. Recall that INLP is an iterative algorithm (\S\ref{sec:inlp}). On the $i$th iteration, the method identifies a single direction $\vec{w_i}$---the parameter vector of a linear classifier---which is predictive of the concept of interest (in our case, RC). The different directions are mutually orthogonal, and after $m$ iterations, the ``concept subspace" is the subspace spanned the rows of the matrix $\matr{W}=[\vec{w_1}^T, \vec{w_2}^T, \dots, \vec{w_m}^T]$. In the $i$th iteration of INLP, the subspace identified so far is \textit{removed} from the representation (by the operation of nullspace projection), and the next classifier $\vec{w_{i+1}}$ is trained to predict the concept over the residual representation. As such, accuracy is expected to decrease with the number of iterations: as the number of iterations increases, the algorithm identifies directions which have a weaker association with the concept. This creates a trade-off between exhaustively -- identifying all the directions which are at least somewhat predictive of the concept, and selectivity -- identifying only directions which have a meaningful association with the concept.

Figure \ref{fig:iters-agreement} presents positive intervention results for different RC-subspace dimensionality on sentences with agreement across RC with attractors; Figure \ref{fig:iters-error} present negative intervention results on sentences on which the model was originally mistaken. Generally, we observe the same trends under all settings, suggesting our method is relatively robust to the dimensionality of the manipulated subspace. In figure \ref{fig:iters-no-agreement} we present the results of intervening on subspaces of different dimensionality, for sentences where we \textit{do not} expect an effect: sentences without attractors, and sentences without RCs. For all contextualized layers we do not see an effect, as expected. For $m=32$ and $m=64$, we see an effect on the uncontextualized embedding layer. This effect may hint towards a spurious information encoded in this uncontexualized layer which is used by the model when predicting agreement, but studying this possibility is beyond the scope of this work.

\subsection{Influence of $\alpha$}
\label{app:alpha}

\begin{figure}[h]
\begin{subfigure}[b]{0.41\textwidth}
   \includegraphics[width=0.95\linewidth]{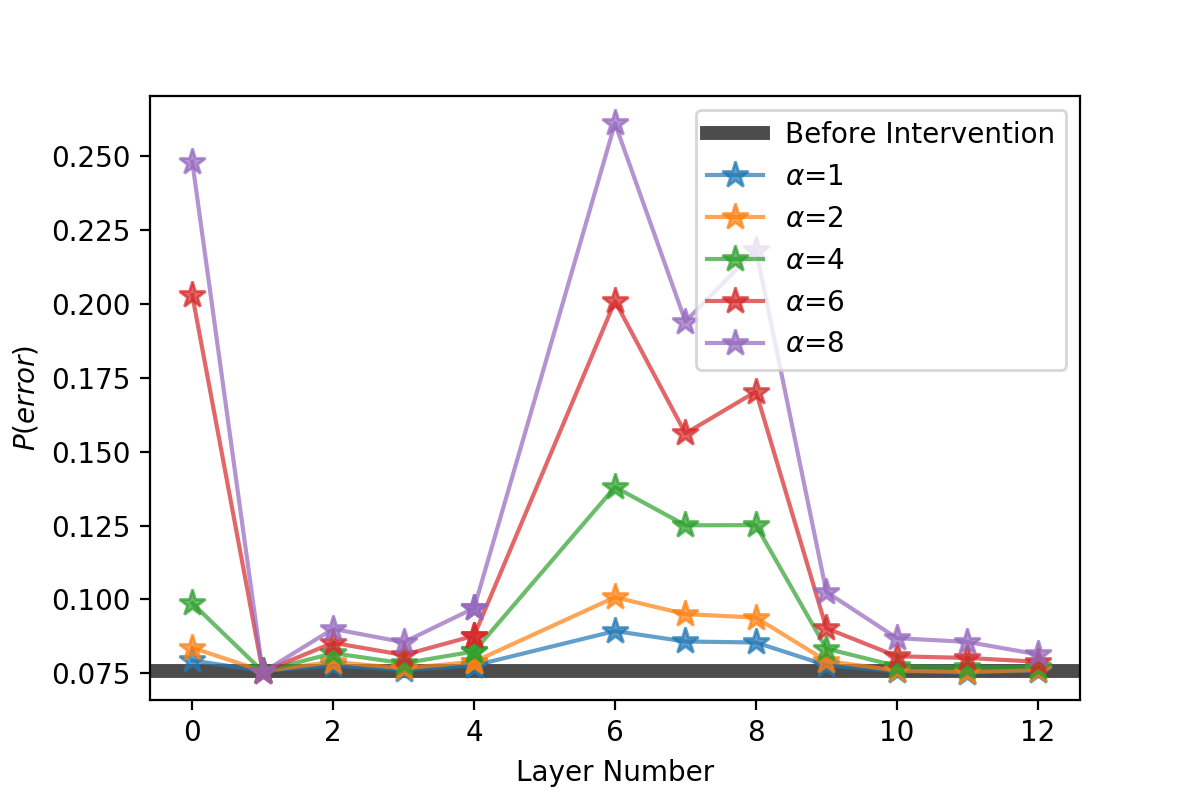}
   \caption{\textbf{Positive} intervention results on sentences with agreement across RC with attractors.}
   \label{fig:alphas-agreement} 
\end{subfigure}

\begin{subfigure}[b]{0.41\textwidth}
   \includegraphics[width=0.95\linewidth]{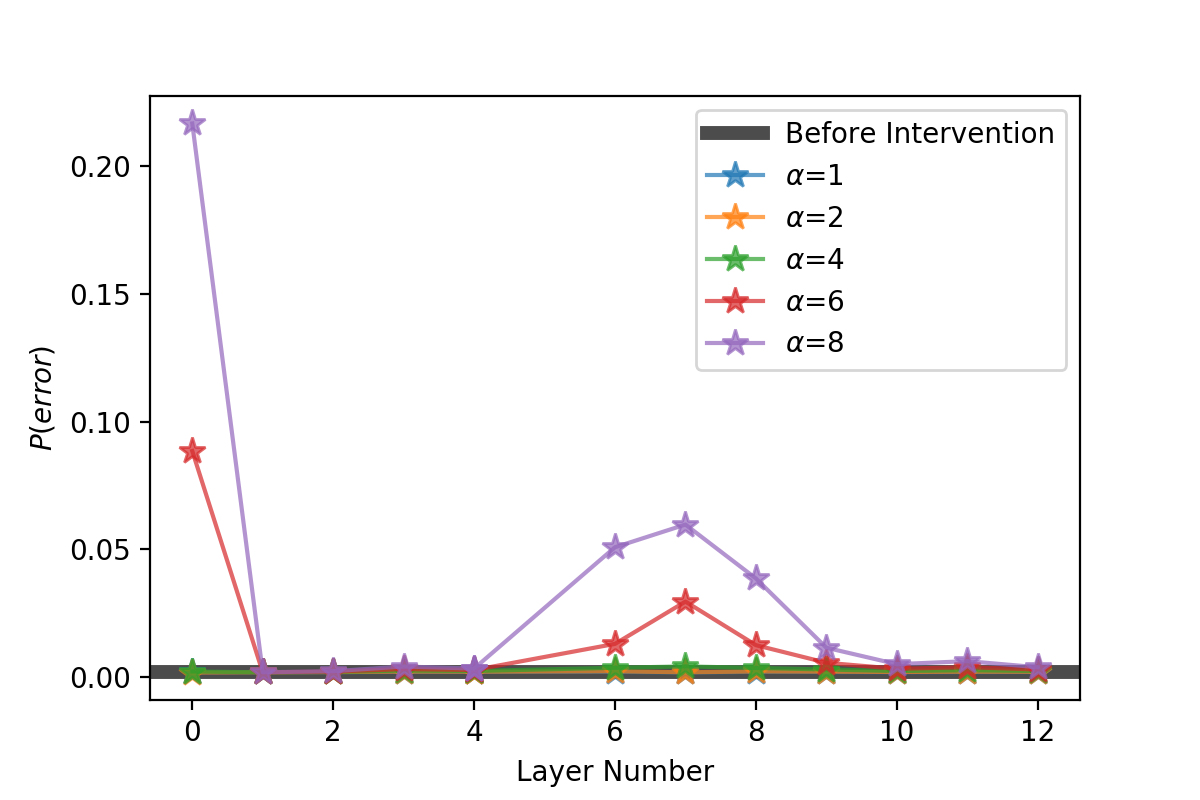}
   \caption{\textbf{Positive} intervention results on sentences with simple agreement and sentences with sentential complements.}
   \label{fig:alphas-no-agreement}
  
\end{subfigure}

\begin{subfigure}[b]{0.41\textwidth}
   \includegraphics[width=0.95\linewidth]{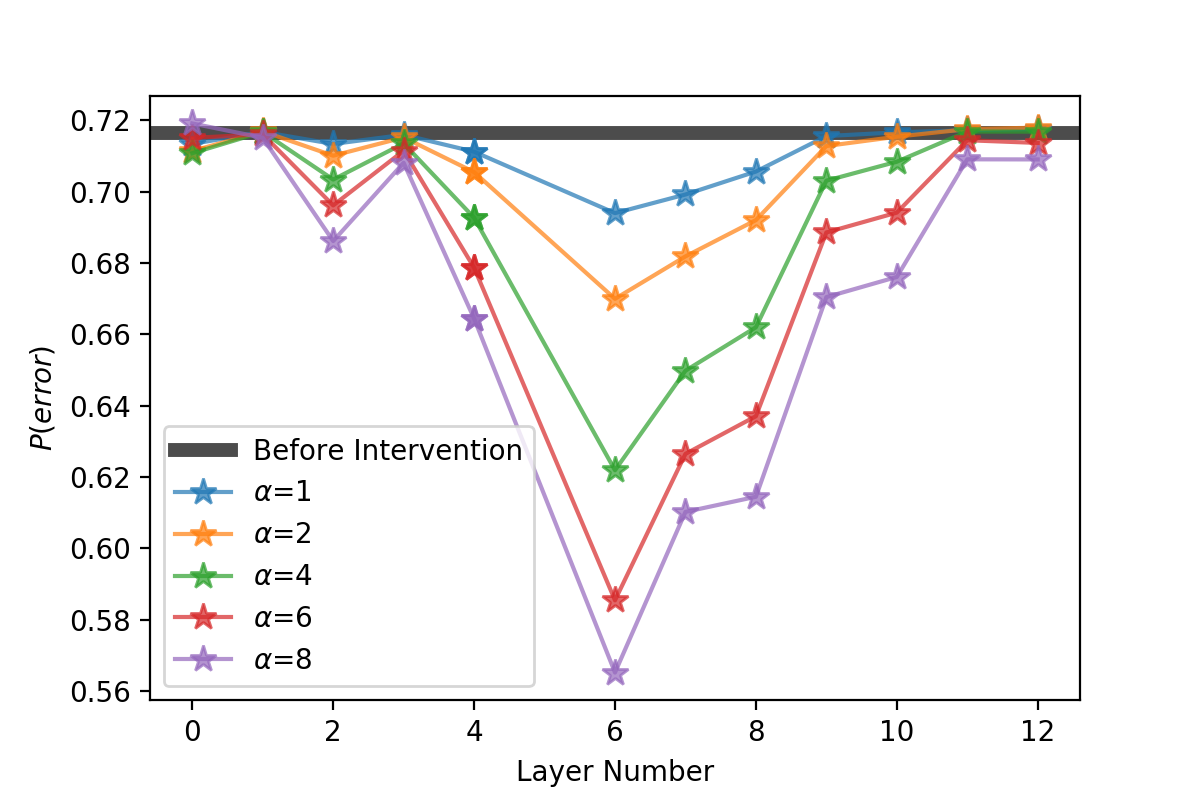}
   \caption{\textbf{Negative} intervention results on sentences with agreement across RC on which the model originally predicted incorrectly.}
   \label{fig:alphas-agreement-neg}
  
\end{subfigure}

\caption[Influence of Alpha]{Influence of $\alpha$ on the probability of error post intervention. }
\end{figure}

In this appendix, we analyze the influence of the parameter $\alpha$ in the \method\ algorithm (Section \S\ref{sec:counterfactual}) on the BERT-base model. Recall that $\alpha$ dictates the step size one takes when calculating the counterfactual mirror image: $\alpha=1$ corresponds to exact mirror image, while $\alpha > 1$  over-emphasizes the RC components over which we take the counterfactual mirror image. 

In Figures \ref{fig:alphas-agreement} and \ref{fig:alphas-no-agreement} we focus on the positive intervention, which is expected to increase the probability of error, making the model act as if the masked verb is within the RC; and in Figure \ref{fig:alphas-agreement-neg} we focus on the negative intervention on sentences on which the model was originally mistaken, which is expected to decrease the probability of error. 

In Figure \ref{fig:alphas-no-agreement} we present the results on the control sentences: sentences without agreement across RC. Overall, the trends we observe are similar for different values of $\alpha$, indicating that \method\ is relatively robust to the value of this parameter. One exception is the large values of $\alpha=8$ and to a lesser degree $\alpha=6$, which result in some increase in the probability of error also in the control sentences, where we do not expect such effect (Figure \ref{fig:alphas-no-agreement}), albeit this increase is much smaller than the increase on sentences with agreement across RC. With a large-enough $\alpha$, the new counterfactual representation might diverge too-much from the distribution of the original representations. Notice that when compared with gradient-based methods for generating counterfactuals \cite{DBLP:journals/corr/abs-2105-14002}, our linear approach has the advantage of being able to control the magnitude of the intervention with a single controlled parameter, which has a clear geometric interpretation: the extent to which one pushes the representations to one direction or another when taking the mirror image.    

\subsection{Influence on Accuracy}
\label{app:accuracy}
\begin{figure}[h]
    \centering
    \includegraphics[width=\columnwidth]{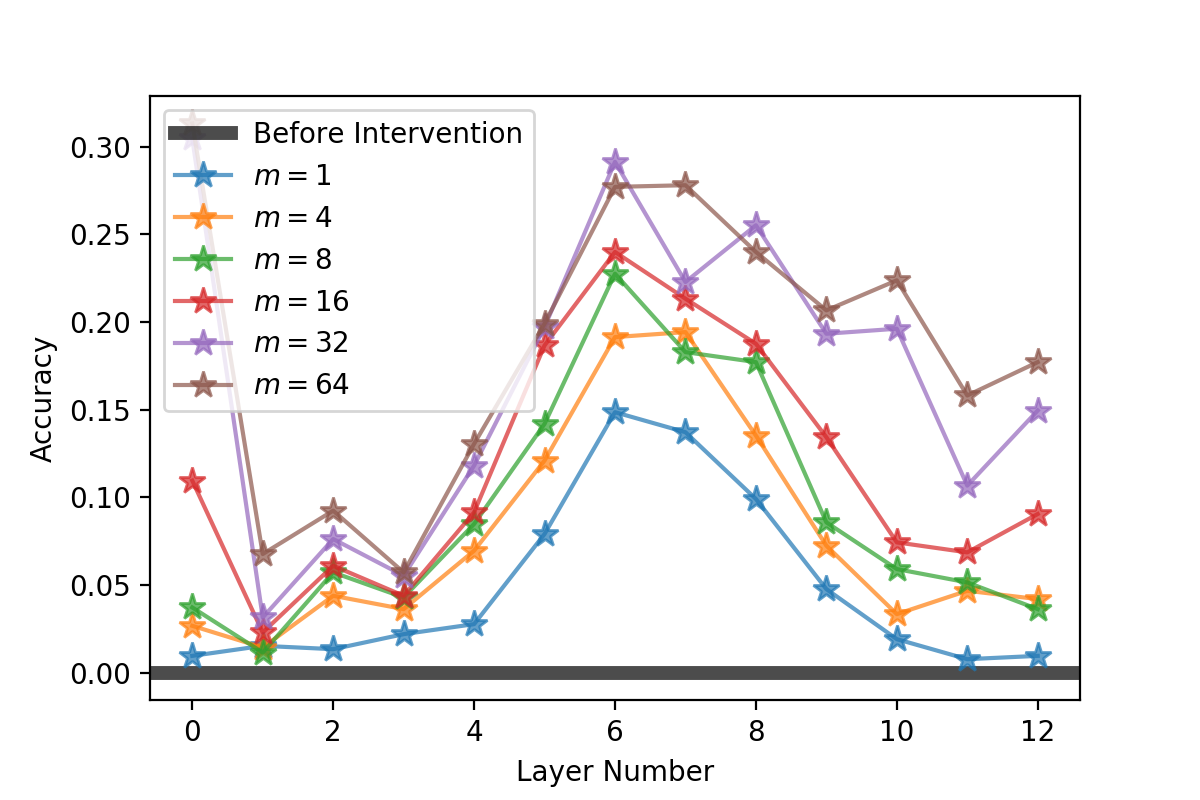}
    \caption{Influence of the \textit{negative intervention} on \textit{accuracy} (the percentage of cases where the model favors the correct form), on sentences on which the model was originally mistaken.}
    \label{fig:fliprate}
\end{figure}

In this appendix, we evaluate the impact of the intervention by its influence on the model's accuracy, calculated as the percentage of cases where the model assigned higher probability to the correct form than to the incorrect form. We focus on the cases on which the model originally predicted incorrectly, Thus, the original accuracy on this group of sentences is 0\%. We use a negative intervention, pushing the model to act as if the verb is (correctly) outside of the RC, which is expected to increase its accuracy.

In Section \S\ref{sec:methods-agreement} we use an alternative measure: probability-of-error. The probability of error is a more sensitive measure, as it might change even when the model's absolute preference for one form over the other has not. However, it is the absolute ranking which eventually dictates the model's top prediction.

Figure \ref{fig:fliprate} presents the results for different dimensionalities of the RC subspace. The trends are similar to the trends shown by the probability-of-error evaluation measure. Notably, in up to ~30\% of the cases, it is possible to flip the model's preference from the incorrect to the correct form solely by manipulating a low-dimensional subspace within the 768-dimensional representation space.  

\end{document}